\documentclass[11pt]{article}

\usepackage[margin=1in]{geometry}
\usepackage[T1]{fontenc}
\usepackage[utf8]{inputenc}
\usepackage{lmodern}
\usepackage[round,authoryear]{natbib}
\usepackage{authblk}

\usepackage{hyperref}
\hypersetup{
  pdftitle={Persistent Cross Entropy},
  pdfauthor={Sijin Yeom and Jae-Hun Jung},
  colorlinks=true,
  linkcolor=blue,
  citecolor=blue,
  urlcolor=blue
}
\usepackage{url}
\usepackage{amsmath,amssymb,mathtools}
\usepackage{amsthm}
\usepackage{bm}
\usepackage{booktabs}
\usepackage{enumitem}
\usepackage{graphicx}
\usepackage{placeins}

\newtheorem{theorem}{Theorem}
\newtheorem{proposition}{Proposition}

\theoremstyle{definition}
\newtheorem{definition}{Definition}
\theoremstyle{plain}
\newtheorem{lemma}{Lemma}

\theoremstyle{remark}

\title{Persistent Cross Entropy}

\author[1]{Sijin Yeom}
\author[1,2]{Jae-Hun Jung}
\affil[1]{Artificial Intelligence Graduate School,
Pohang University of Science and Technology, Pohang, Korea}
\affil[2]{Department of Mathematics,
Pohang University of Science and Technology, Pohang, Korea}
\date{}

\begin{document}

\maketitle

\begin{abstract}
Persistent entropy is the Shannon entropy of a persistence-based probability
measure defined on a persistence diagram. However, its cross-entropy version
is not naturally defined because two persistence diagrams generally have
different event spaces. To bridge these event spaces, we combine a similarity
function with persistence weighting to define an induced probability. The
induced probability reflects information from one diagram on the event space
of the other diagram and assigns unexplained probability mass to the unexplained event.
Using the induced probability, we extend cross entropy to persistence diagrams, called persistent cross entropy
(PCE). We establish the main properties of both the induced probability and
PCE and prove stability theorems for both. Through three numerical studies, we
show that PCE distinguishes diagrams with the same persistent entropy,
separates causal directions in dynamical systems without constructing a joint
persistent diagram, and can be used as a directional topology loss for knowledge
distillation.
\end{abstract}

% ============================================================
\section{Introduction}
\label{sec:introduction}
% ============================================================
Persistent homology extracts multiscale topological features from data by tracking how connected components, cycles, and higher-dimensional holes appear and disappear across a filtration~\citep{edelsbrunner2002topological}. It represents their birth--death structure through persistence diagrams, which provide compact descriptors of the topological structure of complex and high-dimensional data~\citep{carlsson2009topology}. An early landmark application of TDA revealed three clinically distinct subgroups of type 2 diabetes~\citep{li2015identification}. More recently, representations of persistence diagrams have enabled the prediction of molecular interaction energies~\citep{townsend2020representation} and the identification of spatial immune-cell patterns in tumor microenvironments~\citep{vipond2021multiparameter}.

Persistence diagrams are multisets of birth--death pairs in $\mathbb{R}^2$ that may contain different numbers of points and whose points have no canonical ordering, making them difficult to use directly in standard statistical and machine-learning pipelines. Consequently, various representations have been developed to encode persistence information as functions or finite-dimensional vectors. These include functional summaries such as Betti curves~\citep{johnson2021instability}, persistence landscapes~\citep{bubenik2015landscapes}, and persistence silhouettes~\citep{chazal2015silhouettes}, as well as image-based summaries such as persistence images~\citep{adams2017persistenceimages}.

Among scalar summaries of persistence diagrams, persistent entropy normalizes feature persistences to a probability mass
function and summarizes it by Shannon entropy~\citep{chintakunta2015entropy}.
Its stability and related entropy-based methods have also been
studied~\citep{atienza2020stability}. A recent preprint studies its
use in phase-transition detection~\citep{rucco2026persistententropy}.

Shannon cross entropy compares two probability measures on a common event
space, but $p_X$ and $p_Y$ are defined on different persistent diagrams $D_X$ and $D_Y$, respectively. Thus the
\emph{naive} expression
$$H(p_X,p_Y)=-\sum\limits_{u\in D_X}p_X(u)\log p_Y(u)$$
cannot be applied directly. To deal with this issue, we construct the \emph{induced probability}
$p_X^Y$ on
$D_X\cup\{\partial\}$, where $\partial$ is the \emph{unexplained event}. For $u\in D_X$, the mass retained at $u$
depends on how similarly $D_X$ and $D_Y$ explain $u$; the
remaining mass is assigned to $\partial$. Viewing
$p_X$ on the same space with $p_X(\partial)=0$, we define
$$
H(p_X,p_X^Y)
=
-\sum\limits_{u\in D_X}p_X(u)\log p_X^Y(u).
$$

We prove that $p_X^Y=p_X$ exactly when every local difference vanishes and
that $p_X^Y(\partial)$ is the total variation distance between $p_X^Y$ and
$p_X$. The cross entropy excess is both
$\operatorname{KL}(p_X \Vert p_X^Y)$ and an expected squared difference under the chosen
similarity function. We also show that the unexplained mass and the entropy
excess vanish as $D_Y$ approaches $D_X$. For the entropy excess, we further
establish joint local stability when both diagrams vary.
Section~\ref{sec:background} provides the required background,
Section~\ref{sec:induced-probabilities} constructs the induced probability and
proves its main properties, and Section~\ref{sec:persistent-cross-entropy}
develops persistent cross entropy and its stability theory. Finally,
Section~\ref{sec:experiments} evaluates PCE through examples with equal
persistent entropy, causality analysis in dynamical systems, and topology-aware
knowledge distillation.

% ============================================================
\section{Background and Setup}
\label{sec:background}
% ============================================================
\subsection{Persistence Diagrams and Coordinates}

Given a finite point cloud $X$ in a metric space, persistent homology constructs
a multiscale topological summary by building a filtration of simplicial
complexes~\citep{edelsbrunner2002topological,carlsson2009topology}
$$
K_{\varepsilon_1}(X)\subseteq K_{\varepsilon_2}(X)\subseteq \cdots .
$$
In a fixed homological dimension, each feature is recorded by its birth scale
$b$ and death scale $d$. The resulting finite off-diagonal multiset of points
$(b,d)$ records the persistence diagram, and $d-b$ is the persistence of the
feature.

Throughout the paper, we use birth--persistence coordinates on the ambient
space $\mathcal{Z}=\mathbb{R}^2$ and write
$$
z=(b,\ell),
\qquad
\ell=d-b.
$$
For any persistence point $v$, we write $v=(b_v,\ell_v)$.
The death coordinate is recovered as $d=b+\ell$, so this change of
coordinates loses no information. Separating birth from persistence makes
explicit both when a feature appears and how long it lasts, while mapping the
diagonal to the horizontal axis $\ell=0$. This representation is also used in
persistence images, template functions, and vectorized persistence
blocks~\citep{adams2017persistenceimages,perea2023template,
chan2022vectorized}. It makes persistence weighting and matching to the
diagonal simpler. We use $\Vert \cdot \Vert$ for the Euclidean norm in these
coordinates.

For diagram distances, we follow the standard convention of adjoining the
diagonal with infinite multiplicity~\citep{cohen2007stability}. In our measure
construction, however, $D$ denotes only the finite off-diagonal multiset. The
diagonal is used only for matching in Wasserstein or bottleneck distances. The
precise definition of the $1$-Wasserstein distance is given in
Appendix~\ref{app:w1-definition}.

\subsection{Persistence Probabilities and Entropy}

We now assign a probability to the finite off-diagonal points of a persistence
diagram. Let $D$ be a nonempty finite persistence diagram, with all sums over
$D$ understood with multiplicity. We associate each $z=(b,\ell)\in D$ with
the persistence probability
$$
p_D(z)
=
\frac{\ell}{\sum\limits_{v\in D}\ell_v}.
$$
When a point is repeated, each occurrence is treated as a separate probability
atom. We suppress occurrence labels in the main text and make them explicit in
Appendix~\ref{app:proof_induced_probability_measure}.
For $D_X$, we write $p_X:=p_{D_X}$. Every off-diagonal point has positive
persistence, so $p_X(u)>0$ for all $u\in D_X$.
This probability assigns more mass to long-lived features and less to
short-lived features, which may represent topological noise
\citep{atienza2019persistent}.
Related constructions encode persistence through landscape heights and
through explicit weights in silhouettes and persistence
images~\citep{bubenik2015landscapes,chazal2015silhouettes,
adams2017persistenceimages}.

Persistent entropy is the Shannon entropy of $p_D$~\citep{chintakunta2015entropy}:
$$
H(p_D)
=
-\sum\limits_{z\in D}p_D(z)\log p_D(z).
$$
Unlike an empirical distribution, $p_D$ is determined by persistence rather
than sampling frequency. Its entropy is small when persistence is concentrated
on a few features and large when it is spread evenly; for $n$ equal weights,
the maximum is $H(p_D)=\log n$.

% ============================================================
\section{Induced Probabilities}
\label{sec:induced-probabilities}
% ============================================================
Because $D_X$ and $D_Y$ generally have different event spaces, $p_X$ and
$p_Y$ cannot be used directly in Shannon cross entropy. To connect the two
event spaces, we combine a similarity function with persistence weighting to
measure how well each diagram explains the points of $D_X$. We then add the
unexplained event $\partial$ to carry the remaining mass, which gives the
probability on $D_X\cup\{\partial\}$ induced by $D_Y$.
Section~\ref{subsec:similarity-function} introduces the
similarity function, Section~\ref{subsec:induced-probability-measure}
constructs the induced probability, and
Section~\ref{subsec:properties-induced-probabilities} establishes its main
properties and stability.

\subsection{Similarity Between Persistence Points}
\label{subsec:similarity-function}
We first define how similarity is measured between two persistence points.

\begin{definition}
\label{def:similarity_function}
A map $k:\mathcal{Z}\times\mathcal{Z}\to(0,1]$ is called a
\textbf{similarity function} if it satisfies:
\begin{enumerate}[label=(\textit{\roman*})]
\item there exist a metric $\rho$ on $\mathcal{Z}$ and a strictly decreasing function
$h:[0,\infty)\to(0,1]$ with $h(0)=1$ and
$\lim\limits_{r\to\infty}h(r)=0$ such that
$$
k(u,v)=h(\rho(u,v));
$$

\item $k$ is Lipschitz in its second argument: there exists $L_0>0$ such that
$$
|k(u,v)-k(u,v')|
\leq
L_0\Vert v-v' \Vert
$$
for all $u,v,v'\in\mathcal{Z}$.
\end{enumerate}
\end{definition}
The birth and persistence coordinates can have different numerical ranges and
roles. We therefore use fixed global scales $\sigma_b,\sigma_\ell>0$, shared
by all points and diagrams. For $u=(b_u,\ell_u)$ and
$v=(b_v,\ell_v)$, a useful choice is the Gaussian similarity
function
$$
\kappa(u,v)
=
e^{
-\frac{1}{2}
\left\{
\left(\frac{b_v-b_u}{\sigma_b}\right)^2
+
\left(\frac{\ell_v-\ell_u}{\sigma_\ell}\right)^2
\right\}
}.
$$
It satisfies Definition~\ref{def:similarity_function}. Moreover, it can also be obtained
from the maximum-entropy density with prescribed coordinate-wise variances
\citep{jaynes1957information,cover2006elements}. The derivation and Lipschitz bound are given in
Appendix~\ref{app:gaussian-similarity}.

\subsection{Probability Construction}
\label{subsec:induced-probability-measure}

The similarity function measures proximity, but persistence must also be
weighted. Using $\ell$ directly allows points with very large persistence to
contribute without bound, so a uniform stability bound is unavailable unless
persistence is bounded in advance. Throughout the paper, we fix a bounded
increasing function $g:[0,\infty)\to[0,\infty)$ with $g(0)=0$, such as
$g(t)=\frac{t}{1+t}$. This allows points with greater persistence to receive
larger weights while keeping their contributions bounded. The regularity
conditions on $g$ and the definition of $L_\phi$ are given in
Appendix~\ref{app:g-regularity}.

\begin{definition}
\label{def:induced_weight}
Let $D_X,D_Y$ be nonempty finite persistence diagrams in $\mathcal{Z}$, and let
$k$ be a similarity function. Fix a global response scale $\tau>0$. For
$u\in D_X$, $v\in D_X\cup D_Y$, and $D\in\{D_X,D_Y\}$, we define the
following:
\begin{enumerate}[label=(\textit{\roman*}),leftmargin=2em,itemsep=0.2em,topsep=0.4em]
\item $\phi_u(v)=k(u,v)g(\ell_v)$, a score for how well $v$ explains $u$.

\item $\alpha_u(D)=\sum\limits_{v\in D}\phi_u(v)$, a score for how well
$D$ explains $u$.

\item $\delta_X^Y(u)=\alpha_u(D_Y)-\alpha_u(D_X)$, the difference in explanatory power at $u$. \footnote{The value
$\delta_X^Y(u)$ is signed, so one might try to retain its sign in the
induced probability. However, because a probability function must be
nonnegative and sum to one, it is difficult to incorporate signed information
naturally. We therefore focus on the absolute size of $\delta_X^Y(u)$ and
assign the weight accordingly.}

\item $w_X^Y(u)=e^{-\frac{\delta_X^Y(u)^2}{2\tau^2}}$, a weight in
$(0,1]$ that increases as $\lvert\delta_X^Y(u)\rvert$ decreases.
\end{enumerate}
\end{definition}
After any stated preprocessing, the coordinate scales in $k$ and the response
scale $\tau$ are fixed across all comparisons rather than recalibrated for
individual diagram pairs.

The product $w_X^Y(u)p_X(u)$ is the mass at $u$ explained by $D_Y$, but
these masses generally sum to less than one. Equality holds only when
$w_X^Y(u)=1$ for every $u\in D_X$. We therefore introduce an unexplained
event to carry the remaining mass.
\begin{definition}
\label{def:induced_probability_measure}
Let $\partial\notin D_X$ denote the \textbf{unexplained event}. The
\textbf{induced probability} on $D_X\cup\{\partial\}$ is
$$
p_X^Y(u)
=
\begin{cases}
w_X^Y(u)p_X(u),
& u\in D_X,\\[0.15cm]
\displaystyle
\sum\limits_{v\in D_X}\bigl(1-w_X^Y(v)\bigr)p_X(v),
& u=\partial.
\end{cases}
$$
One can regard $p_X$ as a probability measure on the same space
$D_X\cup\{\partial\}$ by setting $p_X(\partial)=0$.
Appendix~\ref{app:proof_induced_probability_measure} verifies that $p_X^Y$
is indeed a probability measure.
\end{definition}

For each $u\in D_X$, the quantity $1-w_X^Y(u)$ determines the fraction of
the mass at $u$ assigned to the unexplained event. Consequently,
$(1-w_X^Y(u))p_X(u)$ is the amount of unexplained probability mass from
$u$. The value $p_X^Y(\partial)$, obtained by summing these amounts over
$D_X$, is therefore the total unexplained mass.

\subsection{Properties of Induced Probabilities}
\label{subsec:properties-induced-probabilities}

Since $p_X^Y$ encodes information from $D_Y$ on the event space of
$D_X \cup \{\partial\}$, the self-induced probability $p_X^X$ should be equal to $p_X$. We
first verify this consistency and then study how the induced probability
changes as the two diagrams become closer, using total variation and the
$1$-Wasserstein distance.

\begin{proposition}
\label{proposition:self-recovery}
The induced probability satisfies the following properties.
\begin{enumerate}[label=(\textit{\roman*}),leftmargin=2em,itemsep=0.4em]
\item Comparing $D_X$ with itself recovers the original probability:
$$
p_X^X=p_X.
$$

\item The induced probability equals $p_X$ exactly when the difference in explanatory power
vanishes on $D_X$:
$$
p_X^Y=p_X
\quad\Longleftrightarrow\quad
\delta_X^Y=0\quad\textnormal{on}\quad D_X.
$$
\item The total variation distance from $p_X$ is exactly the unexplained
mass:
$$
d_{\mathrm{TV}}(p_X^Y,p_X)=p_X^Y(\partial).
$$
\end{enumerate}
\end{proposition}
The proof is given in Appendix~\ref{app:proof-self-recovery}.

We next use the $1$-Wasserstein distance to control the unexplained mass and
changes in the induced probability.

\begin{theorem}
\label{thm:unexplained-mass-stability}
Let $D_X, D_Y$ be nonempty finite persistence diagrams.
There exists $C_X>0$ so that the induced probability satisfies the following properties.
\begin{enumerate}[label=(\textit{\roman*}),leftmargin=2em,itemsep=0.4em]
\item The unexplained mass vanishes as $D_Y$ approaches $D_X$:
$$
p_X^Y(\partial)\leq
C_X W_1(D_X,D_Y)^2.
$$

\item For fixed $D_X$, the induced probability is stable under changes in
$D_Y$; for any nonempty finite persistence diagram $D_{Y'}$,
$$
d_{\mathrm{TV}}(p_X^Y,p_X^{Y'})
\leq
C_X W_1(D_Y,D_{Y'}).
$$
\end{enumerate}
\end{theorem}
The proof is given in
Appendix~\ref{app:proof-unexplained-mass-stability}.

% ============================================================
\section{Persistent Cross Entropy}
\label{sec:persistent-cross-entropy}
% ============================================================
The probabilities $p_X$ and $p_X^Y$ are defined on the same space
$D_X\cup\{\partial\}$. Since $p_X(\partial)=0$, the unexplained event does not
contribute to their cross entropy.
\begin{definition}
\label{def:persistent_cross_entropy}
The \textbf{persistent cross entropy} of $D_X$ relative to $D_Y$ is
$$
H(p_X,p_X^Y)
=
-\sum\limits_{u\in D_X}p_X(u)\log p_X^Y(u).
$$
Its \textbf{entropy excess} is
$\Delta H_X^Y=H(p_X,p_X^Y)-H(p_X)$.
\end{definition}

The entropy excess measures differences in how $D_X$ and $D_Y$ explain the
points of $D_X$ under the chosen similarity function and persistence weighting.
It is directional because the reverse comparison uses $p_Y$ on
$D_Y\cup\{\partial\}$; hence
$H(p_X,p_X^Y)\neq H(p_Y,p_Y^X)$ in general.

\begin{proposition}
\label{proposition:cross_entropy_kl_decomposition}
For every pair of nonempty finite persistence diagrams,
$$
\Delta H_X^Y
=
H(p_X,p_X^Y)-H(p_X)
=
\operatorname{KL}(p_X \Vert p_X^Y)
=
\frac{1}{2\tau^2}\mathbb{E}_{u\sim p_X}\!\left[\delta_X^Y(u)^2\right]
\geq 0.
$$
\end{proposition}
The proof is given in Appendix~\ref{app:proof-cross-entropy-kl}.

Together with Proposition~\ref{proposition:self-recovery}, this identity shows
that the entropy excess vanishes exactly when
$p_X^Y=p_X$; in particular,
$H(p_X,p_X^X)=H(p_X)$.

We next show that the entropy excess vanishes as $D_Y$ approaches $D_X$ in
the $1$-Wasserstein distance.

\begin{theorem}
\label{thm:persistent-cross-entropy-wasserstein-bound}
There exists a single constant $C>0$ such that, for all nonempty finite
persistence diagrams $D_X,D_Y$,
$$
0
\leq
\Delta H_X^Y
\leq
CW_1(D_X,D_Y)^2.
$$
The constant $C$ depends only on the similarity function $k$ in
Definition~\ref{def:similarity_function} and on $g$ and $\tau$ fixed in
Section~\ref{subsec:induced-probability-measure}.
\end{theorem}
The proof is given in Appendix~\ref{app:proof-pce-wasserstein-bound}.

We finally establish joint local stability when both diagrams vary.

\begin{theorem}
\label{thm:pce-joint-local-stability}
Fix nonempty finite persistence diagrams $D_X,D_Y$. The entropy excess is
locally stable in both diagrams: there exists a finite constant $C_{X,Y}$ such that, whenever
$D_{X'}$ and $D_{Y'}$ are sufficiently close to $D_X$ and $D_Y$, respectively,
$$
|\Delta H_X^Y-\Delta H_{X'}^{Y'}|
\leq
C_{X,Y}
\left(
W_1(D_X,D_{X'})+W_1(D_Y,D_{Y'})
\right).
$$
The same constant $C_{X,Y}$ applies throughout this neighborhood and does not
depend on $D_{X'}$ or $D_{Y'}$.
In particular, if $D_{X'}=D_X$, then
$$
|\Delta H_X^Y-\Delta H_X^{Y'}|
\leq
C_{X,Y}W_1(D_Y,D_{Y'}).
$$
The complete proof, including an explicit neighborhood and a choice of
$C_{X,Y}$, is given in
Appendix~\ref{app:proof-pce-joint-local-stability}.
\end{theorem}

\section{Experiments}
\label{sec:experiments}

We evaluate PCE through three experiments.
Section~\ref{sec:two-loops-equal-pe} tests whether PCE can distinguish diagrams
with the same persistent entropy. Section~\ref{sec:spring-mass-benchmark}
studies causality in dynamical systems and compares PCE with
symmetric persistence-diagram methods that require a joint reconstruction.
Section~\ref{sec:topology-aware-kd} then applies PCE to knowledge distillation
and tests whether its teacher-to-student direction is useful for transferring
topological information. Full numerical settings are given in
Appendix~\ref{app:additional-experiments}.

\subsection{Two loops with equal persistent entropy}
\label{sec:two-loops-equal-pe}

Persistent entropy summarizes the probability distribution associated with a
single persistence diagram by one scalar. Consequently, diagrams with clearly
different feature compositions can have the same persistent entropy. To
illustrate this limitation, we construct five point clouds. The point cloud
$X$ contains a large loop, a small loop, and several noise components.
The point cloud $Y_1$ contains the two loops without the surrounding noise,
$Y_2$ retains only the large loop, $Y_3$ retains only the small loop, and
$Y_4$ contains only noise. Their parameters are adjusted so that the
persistent entropies of their $H_1$ persistence diagrams all equal $1.500$
when rounded to three decimal places.

Although all five persistent entropies round to $1.500$, PCE for
$Y_1$ through $Y_4$ equals $1.752$, $2.281$, $2.997$, and
$3.977$, respectively. Their unexplained masses are $0.194$, $0.498$,
$0.751$, and $0.770$. Thus PCE separates feature compositions that
persistent entropy cannot distinguish. Since $D_X$ is fixed in all four
comparisons, the PCE values share the same baseline $H(p_X)$, and their
differences are exactly the differences in entropy excess. The entropy excess
measures the average severity of the difference, whereas
$p_X^Y(\partial)$ is the fraction of the probability mass of $D_X$ left
unexplained. The $Y_3$ comparison illustrates why both parts of the augmented
probability are useful. Because $Y_3$ contains a small loop, the
corresponding point in $D_X$ receives the largest probability among the
points of $D_X$. Nevertheless, $Y_3$ leaves $0.751$ of the total mass at
$\partial$. Since the retained mass is not renormalized on $D_X$, the
small-loop point is not assigned an artificially inflated probability: it
remains the most strongly explained feature of $D_X$ while the large amount
of unexplained structure is preserved explicitly.

\begin{figure}[!t]
    \centering
    \includegraphics[width=\textwidth]{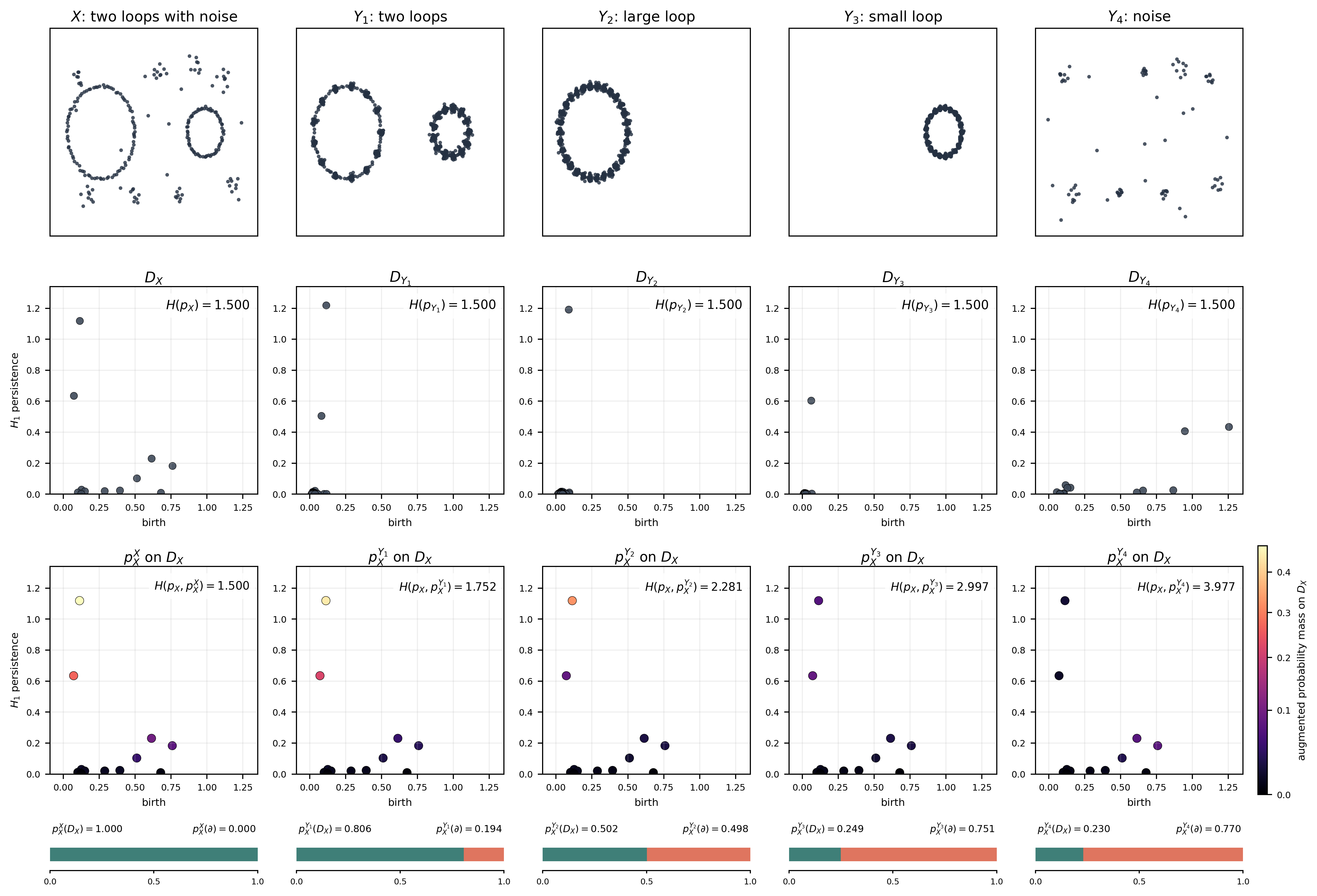}
    \caption{Two-loop experiment with equal persistent entropy. The first row
    shows the point clouds $X,Y_1,\ldots,Y_4$, and the second row shows
    $D_X,D_{Y_1},\ldots,D_{Y_4}$; the square and triangle in $D_X$ mark the
    large- and small-loop features, respectively, and circles denote the
    remaining noise features. Each diagram is annotated with its persistent
    entropy. The third row shows the actual augmented probability masses
    $p_X^X(u),p_X^{Y_1}(u),\ldots,p_X^{Y_4}(u)$ for $u\in D_X$ on a common
    color scale and reports the corresponding PCE. These masses are not
    renormalized within $D_X$. The bar below each panel completes the
    probability by partitioning unit mass between $D_X$ and the unexplained
    event $\partial$.}
    \label{fig:two-loops-equal-pe}
\end{figure}

\subsection{Spring--mass benchmark}
\label{sec:spring-mass-benchmark}

We study causality between deterministic dynamical systems using
the spring--mass benchmark of \citet{DBLP:journals/jsiaml/BandoKY22},
$$
\ddot{x}=-x+0.7\alpha y,
\qquad
\ddot{y}=-0.7y+\beta x,
$$
where $\alpha$ controls the influence $B\to A$ and $\beta$ controls
$A\to B$. The coupling grid contains independent, one-way, and
bidirectional regimes. Following Takens' embedding theorem, we turn each
scalar time series into a point cloud of delay vectors, such as
$X_A=\{(x_t,x_{t+q},\ldots,x_{t+(E-1)q})\}_t$, which reconstructs the
attractor of the observed dynamical system~\citep{takens1981detecting}. Let
$D_A$ and $D_B$ be the persistence diagrams of the two reconstructed
attractors. A symmetric distance $d(D_A,D_B)$ gives only one value and
therefore cannot determine a direction. Such a distance requires an
additional common object for directional comparison. Following the
multivariate delay reconstruction of \citet{cao1998dynamics},
\citet{DBLP:journals/jsiaml/BandoKY22} construct a joint attractor from both
time series for this purpose. We follow their experimental setup and denote
the persistence diagram of this joint attractor by $D_{AB}$. Direction is
then assessed by comparing
$d(D_{AB},D_A)$ and $d(D_{AB},D_B)$. Under $A\to B$, the reconstruction
$X_B$ is expected to represent the total attractor, so
$D_B$ should be closer to $D_{AB}$ than $D_A$ is; the roles are reversed
under $B\to A$~\citep{DBLP:journals/jsiaml/BandoKY22}. The full
reconstruction is described in
Appendix~\ref{app:spring-mass-configuration}.

We compare PCE with three groups of methods, shown in the same order in
Figure~\ref{fig:spring-mass}. First, bottleneck and Wasserstein distances
compare diagrams through optimal matchings. The bottleneck distance records
the largest matched displacement, whereas the Wasserstein distance aggregates
the matched displacements~\citep{cohen2007stability,mileyko2011probability}.
Second, Betti curves count the features alive across the filtration
\citep{johnson2021instability}; persistence landscapes organize the features
into layers of piecewise-linear functions~\citep{bubenik2015landscapes}; and
persistence silhouettes form a persistence-weighted average of these
functions~\citep{chazal2015silhouettes}. Finally, the persistence scale-space
kernel and persistence images give smoothed representations on the diagram
plane. The former uses heat diffusion and a Hilbert-space distance
\citep{reininghaus2015stable}, while the latter places weighted Gaussian
functions on the birth--persistence plane and discretizes the result
\citep{adams2017persistenceimages}.

Unlike these baselines, PCE directly produces the directed pair
$(\Delta H_B^A,\Delta H_A^B)$ from $D_A$ and $D_B$, without constructing
$D_{AB}$. The entropy excess $\Delta H_X^Y$ measures the cost of using
$D_Y$ to explain $D_X$. Under $A\to B$,
system $B$ carries information from $A$ together with its own response.
Consequently, $D_B$ explains $D_A$ at low cost, whereas $D_A$ does not
fully explain $D_B$: $\Delta H_A^B$ is small and $\Delta H_B^A$ is
large. This places the $A\to B$ cases in the lower-right half of the first
panel. The interpretation is reversed for $B\to A$, whose points lie in the
upper-left half. Figure~\ref{fig:spring-mass} shows that all $16$ one-way
cases lie in the expected half-plane without constructing $D_{AB}$.

The unexplained mass gives a bounded and more direct view of the same
asymmetry. Specifically, $p_X^Y(\partial)$ is the fraction of the
persistence mass of $D_X$ not explained by $D_Y$. For $A\to B$, the
median pair $(p_B^A(\partial),p_A^B(\partial))$ is $(0.888,0.024)$.
Thus, $88.8\%$ of the persistence mass of $D_B$ is not explained by
$D_A$, while only $2.4\%$ of the persistence mass of $D_A$ is not
explained by $D_B$. For $B\to A$, the median pair is
$(0.014,0.883)$, and the interpretation is reversed. The entropy excess
measures the average severity of the difference, whereas the unexplained mass
reports its share of the persistence probability mass. The independent case
has almost equal values in both directions, so it shows no directional
preference; their small size alone should not be read as evidence of coupling.

\begin{figure}[t]
    \centering
    \includegraphics[width=\textwidth]{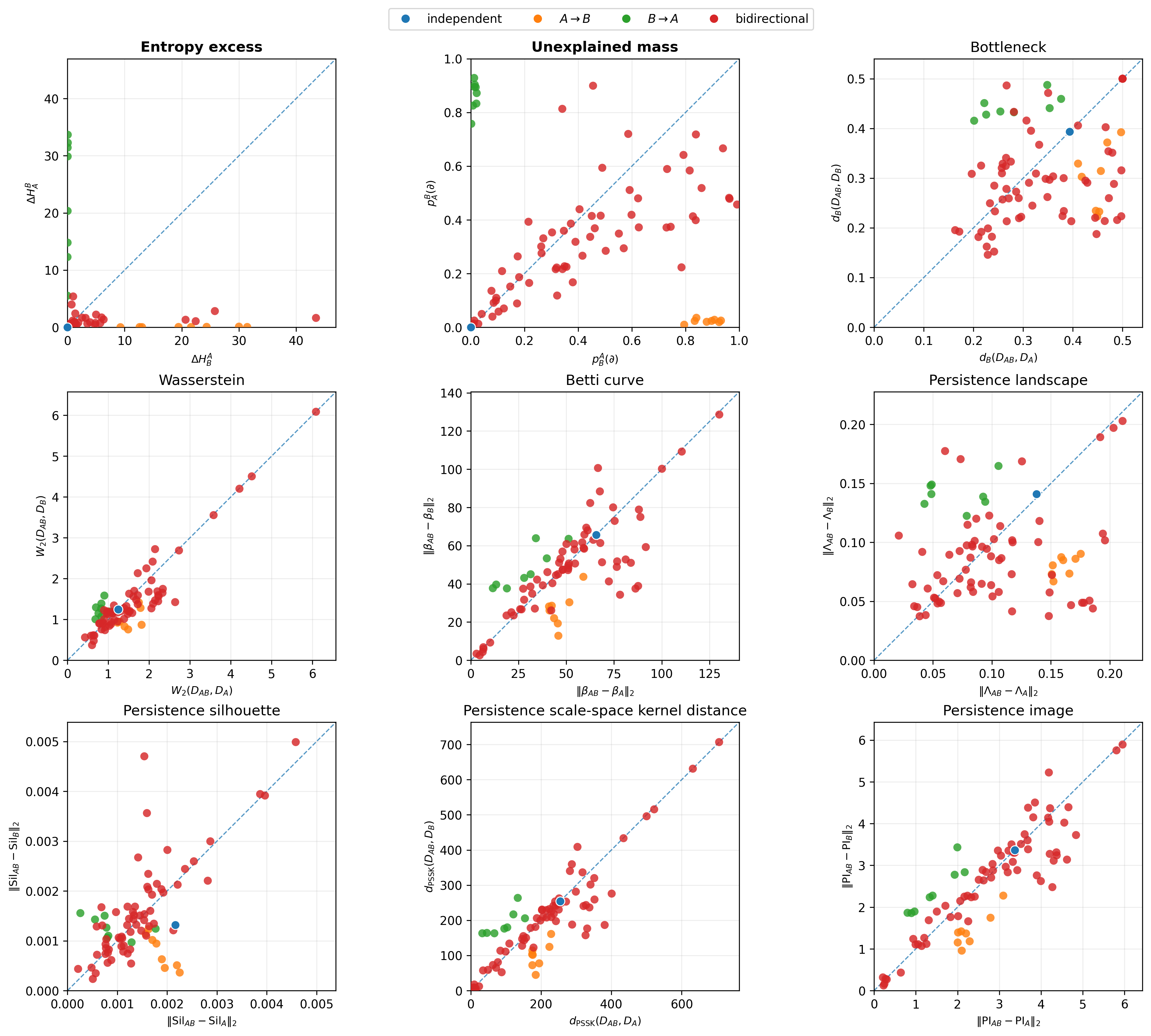}
    \caption{Spring--mass directional comparison. Each point represents one
    coupling pair: blue denotes the independent case, orange denotes
    $A\to B$, green denotes $B\to A$, and red denotes bidirectional
    coupling. The dashed line is $y=x$; under the common plotting
    convention, the lower-right half favors $A\to B$, the upper-left half
    favors $B\to A$, and points near the diagonal show little directional
    preference. The two PCE-based panels are computed only from $D_A$ and
    $D_B$. They plot $(\Delta H_B^A,\Delta H_A^B)$ and
    $(p_B^A(\partial),p_A^B(\partial))$, respectively. Each coordinate in
    the second panel is the fraction of persistence mass not explained in that
    direction. Every remaining panel uses the joint diagram and plots
    $(d(D_{AB},D_A),d(D_{AB},D_B))$. In order, the panels use bottleneck
    distance, $2$-Wasserstein distance, $L^2$ distances between Betti
    curves, persistence landscapes, and persistence silhouettes, the
    persistence scale-space kernel distance, and Euclidean distance between
    persistence images.}
    \label{fig:spring-mass}
\end{figure}

To compare the nine panels numerically, we treat the two plotted coordinates
of each method as a two-dimensional representation and measure regime
separation by PERMANOVA $R^2$ with Euclidean distance
\citep{anderson2001nonparametric}. The four-regime evaluation uses all
$81$ settings with the independent, $A\to B$, $B\to A$, and
bidirectional labels. The two-direction evaluation uses only the balanced set
of $16$ one-way cases. These scores measure between-group separation
relative to total variation; they are not classification accuracies. In
particular, the four-regime result is descriptive because the group sizes are
$1,8,8,$ and $64$.

\begin{table}[t]
    \caption{Separation of the coupling regimes in the two-dimensional
    representations shown in Figure~\ref{fig:spring-mass}. Larger PERMANOVA
    $R^2$ indicates stronger separation. The best result in each evaluation
    is shown in bold.}
    \label{tab:spring-mass-separation}
    \centering
    \small
    \begin{tabular}{lcc}
        \toprule
        Method
        & Four regimes $R^2$
        & Two directions $R^2$ \\
        \midrule
        Entropy excess
        & $\mathbf{0.560}$
        & $0.740$ \\
        Unexplained mass
        & $0.449$
        & $\mathbf{0.993}$ \\
        Bottleneck
        & $0.180$
        & $0.733$ \\
        Wasserstein
        & $0.050$
        & $0.662$ \\
        Betti curve
        & $0.091$
        & $0.513$ \\
        Persistence landscape
        & $0.197$
        & $0.876$ \\
        Persistence silhouette
        & $0.082$
        & $0.616$ \\
        Persistence scale-space kernel
        & $0.082$
        & $0.678$ \\
        Persistence image
        & $0.107$
        & $0.508$ \\
        \bottomrule
    \end{tabular}
\end{table}

Entropy excess gives the strongest four-regime separation, followed by
unexplained mass. For the balanced one-way comparison, unexplained mass
reaches $R^2=0.993$, followed by the persistence landscape at $0.876$.
Entropy excess reaches $0.740$, slightly above the bottleneck distance at
$0.733$. Thus the two PCE quantities provide the strongest overall
four-regime separation, while unexplained mass also gives the clearest
separation between the two one-way directions. Both are computed without the
joint diagram $D_{AB}$.

The spring--mass experiment shows that PCE can identify causal direction in a
linear dynamical system. Our preliminary tests on nonlinear and chaotic
systems produced mixed results: PCE identified the
causal direction in some systems but not in others. Methods based on a joint
attractor showed similarly mixed performance, indicating that nonlinear and
chaotic dynamics remain challenging for TDA-based causality analysis in
general. Further study is needed to understand when PCE works well and how it
can be improved for these systems.

\FloatBarrier

\subsection{Topology-aware knowledge distillation}
\label{sec:topology-aware-kd}

Topology-aware knowledge distillation transfers the global structure of
intermediate features in addition to class probabilities. TopKD approximates
the persistence images of teacher and student feature clouds with a frozen
RipsNet~\citep{pmlr-v196-surrel22a} and minimizes the symmetric loss
$\mathcal L_{\mathrm{Top}}=
\Vert \widehat{\mathrm{PI}}_T-\widehat{\mathrm{PI}}_S \Vert_2^2$
\citep{kim2024topkd}. In knowledge distillation, however, the two networks have
different roles: the student should reproduce the teacher's information. We
therefore replace this symmetric topology loss with PCE-based objectives
directed from the teacher to the student.

For each mini-batch, we $\ell_2$-normalize the teacher and student features and
compute their exact finite $H_0$ Vietoris--Rips persistence diagrams,
$D_T$ and $D_S$. The student diagram $D_S$ induces the probability $p_T^S$
on the teacher diagram $D_T$. We compare two teacher-to-student topology
terms. PCE uses the augmented probability and keeps the unexplained event.
As an experimental variant, we also consider \textbf{explained-mass PCE}
(EM-PCE), which removes the unexplained event and renormalizes only the mass
explained on $D_T$. Thus, PCE measures both relative changes among teacher
features and absolute unexplained mass, whereas EM-PCE focuses on the relative
allocation among the explained teacher features. Their training objectives are
$$
\mathcal L_{\mathrm{PCE}}
=
\mathcal L_{\mathrm{CE}}
+2\mathcal L_{\mathrm{KD}}
+c_{\mathrm{batch}}\Delta H_T^S,
\qquad
\mathcal L_{\mathrm{EM\text{-}PCE}}
=
\mathcal L_{\mathrm{CE}}
+2\mathcal L_{\mathrm{KD}}
+5\Delta H_T^{S,\mathrm{EM}}.
$$
Both terms use the same direction because the student should reproduce the
teacher topology: $D_S$ is used to explain $D_T$, while the reverse quantity
$H_S^T$ would target the opposite transfer. For PCE, we separate the loss
direction from its optimization scale by matching its feature-gradient norm to
that of the TopKD topology term with a detached batch-wise coefficient
$c_{\mathrm{batch}}$. EM-PCE instead uses the fixed topology coefficient $5$
from TopKD. PCE is defined in Section~\ref{sec:persistent-cross-entropy}, while
the EM-PCE probability and the remaining implementation details are given in
Appendix~\ref{app:knowledge-distillation-configuration}. The frozen RipsNet is
still evaluated in these runs, so this experiment compares the topology
objectives and is not intended as a comparison of computational cost.

We use CIFAR-100~\citep{krizhevsky2009learning} with a ResNet56 teacher and a
ResNet20 student~\citep{he2016deep}, following the TopKD training schedule.
The three runs use seed $7$, the same student initialization, mini-batch order,
augmentation, optimizer, and CE and KD terms; only the topology objective
differs. The final PCE configuration was fixed using a deterministic validation
split made only from the CIFAR-100 training set. The official test set was
evaluated after the configuration was fixed. We report the final accuracy and
the mean over the last $10$ epochs, without selecting a checkpoint by test
accuracy. Full settings are given in
Appendix~\ref{app:knowledge-distillation-configuration}.

\begin{table}[h]
    \caption{Seed-$7$ CIFAR-100 Top-1 accuracy (\%). Final is the accuracy at
    epoch $240$, and Last-10 mean is the mean over epochs $231$--$240$.}
    \label{tab:topkd-pce-results}
    \centering
    \small
\begin{tabular}{lcc}
        \toprule
        Method & Final & Last-10 mean \\
        \midrule
        TopKD & $70.99$ & $71.000$ \\
        PCE & $71.25$ & $71.223$ \\
        EM-PCE & $\mathbf{71.49}$ & $\mathbf{71.606}$ \\
        \bottomrule
    \end{tabular}
\end{table}

EM-PCE gives the highest final accuracy and last-$10$-epoch mean in this
seed-$7$ experiment. Relative to TopKD, it improves these values by $0.50$ and
$0.606$ percentage points, respectively. PCE also improves both values, but by
smaller amounts. These results suggest that the relative
allocation of the explained teacher mass can provide a useful optimization
signal even without directly penalizing the unexplained mass. Because EM-PCE
uses a fixed coefficient while PCE uses batch-wise gradient matching, the
current experiment compares the implemented objectives rather than isolating
the effect of explained-mass normalization alone. During PCE training, the
epoch-averaged unexplained mass decreases overall from $0.2494$ in the first
epoch to $0.00756$ in the final epoch. This indicates that the student diagram
explains more of the teacher persistence mass by the end of training. We treat
this overall decrease as a diagnostic of topology matching, not as evidence
that it directly causes the accuracy gain.

\FloatBarrier

\subsection*{AI use statement}

OpenAI Codex was used in preparing this work. In the theoretical development,
it was used primarily to edit language, improve notation and mathematical
typesetting, identify inconsistencies, and check derivations and proofs; the
definitions, mathematical claims, and final arguments were determined and
verified by the authors. It was used more extensively in the experimental
work to inspect and modify code, implement the augmented probability
construction, rerun experiments, generate figures, and verify numerical
identities and test results. It also assisted with literature searches and
manuscript editing. The authors reviewed all AI-assisted material and take full
responsibility for the final content of the paper.

% ============================================================
% Bibliography
% ============================================================
\bibliography{references}
\bibliographystyle{plainnat}

% ============================================================
% Appendix
% ============================================================
\appendix

\section{Proofs for Probabilities Induced on the Augmented Event Space}

\subsection{Regularity assumptions on \texorpdfstring{$g$}{g}}
\label{app:g-regularity}

The uniform stability results use a function
$g:[0,\infty)\to[0,\infty)$ satisfying
$$
g(0)=0,
\qquad
0\leq g(t)\leq M_g
\quad\text{for every }t\ge0,
$$
for some finite constant $M_g$. We further assume that $g$ is
monotonically increasing and globally Lipschitz: there exists $L_g<\infty$ such that
$$
|g(s)-g(t)|
\leq
L_g|s-t|
\qquad
\text{for every }s,t\ge0.
$$
The condition $g(0)=0$ makes a point contribute zero when it is matched to
the diagonal. Boundedness and global Lipschitz continuity make the constants
in the stability estimates independent of the diagrams and their maximum
persistence. Together with the constant $L_0$ from
Definition~\ref{def:similarity_function}, we define
$$
L_\phi=L_g+M_gL_0.
$$
The function used in the experiments satisfies these assumptions:
$$
g(t)=\frac{t}{1+t}.
$$
\subsection{Maximum-entropy derivation of the Gaussian similarity function}
\label{app:gaussian-similarity}
The Gaussian uniquely maximizes differential entropy among densities with a
fixed mean and covariance~\citep{cover2006elements}. Set
$\mathbb{E}_q[U]=u$ and
$\Sigma=\operatorname{diag}(\sigma_b^2,\sigma_\ell^2)$. The maximizer is
$$
q_u(v)=\frac{1}{2\pi\sigma_b\sigma_\ell}
e^{-\frac{1}{2}(v-u)^\top\Sigma^{-1}(v-u)}.
$$
For $u=(b_u,\ell_u)$ and $v=(b_v,\ell_v)$, the quadratic form is
$$
(v-u)^\top\Sigma^{-1}(v-u)
=
\left(\frac{b_v-b_u}{\sigma_b}\right)^2
+
\left(\frac{\ell_v-\ell_u}{\sigma_\ell}\right)^2.
$$
The density is largest at $v=u$, where
$$
q_u(u)=\frac{1}{2\pi\sigma_b\sigma_\ell}.
$$
Normalizing the density by this maximum removes the dimensional prefactor and
gives
\begin{align*}
\frac{q_u(v)}{q_u(u)}
&=
e^{-\frac{1}{2}(v-u)^\top\Sigma^{-1}(v-u)}\\
&=
e^{
-\frac{1}{2}
\left\{
\left(\frac{b_v-b_u}{\sigma_b}\right)^2
+
\left(\frac{\ell_v-\ell_u}{\sigma_\ell}\right)^2
\right\}
}\\
&=
\kappa(u,v).
\end{align*}

Let
$$
\rho(u,v)^2
=
\left(\frac{b_u-b_v}{\sigma_b}\right)^2
+
\left(\frac{\ell_u-\ell_v}{\sigma_\ell}\right)^2.
$$
This is the Euclidean metric after coordinate
rescaling, and $\kappa(u,v)=h(\rho(u,v))$ for
$h(r)=e^{-\frac{r^2}{2}}$, which has all properties required in
Definition~\ref{def:similarity_function}. In particular, $h(0)=1$, $h$ is
strictly decreasing on $(0,\infty)$, and $h(r)$ approaches zero as $r$
approaches infinity.

It remains to verify the Lipschitz condition. Differentiating with respect to
the second argument gives
$$
\frac{\partial\kappa(u,v)}{\partial b_v}
=
-\frac{b_v-b_u}{\sigma_b^2}\kappa(u,v),
\qquad
\frac{\partial\kappa(u,v)}{\partial\ell_v}
=
-\frac{\ell_v-\ell_u}{\sigma_\ell^2}\kappa(u,v).
$$
Let $\sigma_{\min}=\min\{\sigma_b,\sigma_\ell\}$. Then
$$
\Vert \nabla_v\kappa(u,v) \Vert
\leq
\frac{\rho(u,v)e^{-\frac{\rho(u,v)^2}{2}}}{\sigma_{\min}}
\leq
\frac{1}{\sigma_{\min}}.
$$
The mean value theorem therefore gives
$|\kappa(u,v)-\kappa(u,v')|\leq
\sigma_{\min}^{-1} \Vert v-v' \Vert$, proving the Lipschitz condition.

\subsection{Verification of Definition~\ref{def:induced_probability_measure}}
\label{app:proof_induced_probability_measure}

For any finite off-diagonal multiset $D$, let $\operatorname{supp}(D)$ be the
set of its distinct coordinates and let $m_D(z)$ be the multiplicity of
$z\in\operatorname{supp}(D)$. Define the occurrence set
$$
\widehat D
=
\left\{
(z,j):z\in\operatorname{supp}(D),\quad
j\in\{1,\ldots,m_D(z)\}
\right\}
$$
and the coordinate projection $\pi_D(z,j)=z$. Thus two occurrences with the
same coordinates remain distinct elements of $\widehat D$. We write
$\widehat D_X$ and $\pi_X$ for the occurrence set and projection associated
with $D_X$.
For $\xi\in\widehat D_X$, the notation used below means
$$
p_X(\xi)
=
\frac{\ell_{\pi_X(\xi)}}
{\sum\limits_{\eta\in\widehat D_X}\ell_{\pi_X(\eta)}},
\qquad
p_X^Y(\xi)
=
w_X^Y(\pi_X(\xi))p_X(\xi).
$$
The unexplained-event mass is correspondingly
$$
p_X^Y(\partial)
=
\sum\limits_{\xi\in\widehat D_X}
\left(1-w_X^Y(\pi_X(\xi))\right)p_X(\xi),
\qquad
p_X(\partial)=0.
$$
Define the finite event space and its sigma-algebra by
$$
\Omega_X
=
\widehat D_X\cup\{\partial\},
\qquad
\mathcal{F}_X
=
2^{\Omega_X}.
$$
First, the occurrence-level persistence masses sum to one:
\begin{align*}
\sum\limits_{\xi\in\widehat D_X}p_X(\xi)
&=
\frac{
\sum\limits_{\xi\in\widehat D_X}\ell_{\pi_X(\xi)}
}{
\sum\limits_{\eta\in\widehat D_X}\ell_{\pi_X(\eta)}
}\\
&=1.
\end{align*}
Every $p_X(\xi)$ is positive, and $0<w_X^Y(\pi_X(\xi))\leq1$.
Therefore all masses assigned by $p_X^Y$ are nonnegative. Their total is
\begin{align*}
&\sum\limits_{\xi\in\widehat D_X}p_X^Y(\xi)
+p_X^Y(\partial)\\
&=
\sum\limits_{\xi\in\widehat D_X}
w_X^Y(\pi_X(\xi))p_X(\xi)
+
\sum\limits_{\xi\in\widehat D_X}
\left(1-w_X^Y(\pi_X(\xi))\right)p_X(\xi)\\
&=
\sum\limits_{\xi\in\widehat D_X}p_X(\xi)\\
&=1.
\end{align*}
Thus the atom masses extend by finite sums to a probability measure
$p_X^Y$ on $(\Omega_X,\mathcal{F}_X)$. The same argument, with
$p_X(\partial)=0$, makes $p_X$ a probability measure on this space.

\medskip

\subsection{\texorpdfstring{$1$-Wasserstein distance}{1-Wasserstein distance}}
\label{app:w1-definition}

For finite persistence diagrams $D$ and $D'$, we use the diagram
$1$-Wasserstein distance with the Euclidean ground metric in the
birth--persistence coordinates fixed in Section~\ref{sec:background}
\citep{mileyko2011probability}. Following the standard convention, the
diagonal $\Delta$ is included with infinite multiplicity. Let
$\Gamma(D,D')$ be the set of bijections between $D\cup\Delta$ and
$D'\cup\Delta$ that have only finitely many nonzero matching costs. Then
$$
W_1(D,D')
=
\inf\limits_{\gamma\in\Gamma(D,D')}
\sum\limits_{z\in D\cup\Delta}
\Vert z-\gamma(z)\Vert.
$$
The bijections match occurrences, so repeated points remain distinct. Matching
an off-diagonal point to $\Delta$ gives its distance to the diagonal and
therefore includes unmatched points in the same sum. In birth--persistence
coordinates, the diagonal is
$\{(t,0):t\in\mathbb{R}\}$. The linear change of coordinates
$(b,d)\mapsto(b,d-b)$ is invertible, so this distance is equivalent to the
usual birth--death version up to fixed multiplicative constants.

\subsection{Proof of Proposition~\ref{proposition:self-recovery}}
\label{app:proof-self-recovery}

We prove the three claims in order. As in
Appendix~\ref{app:proof_induced_probability_measure}, repeated points are
treated as distinct atoms in the occurrence set $\widehat D_X$, and
$\pi_X:\widehat D_X\to D_X$ returns the coordinate of each occurrence.

\medskip
\noindent\textit{Proof of (i).}
For every $u\in D_X$, Definition~\ref{def:induced_weight} gives
\begin{align*}
\delta_X^X(u)
&=
\alpha_u(D_X)-\alpha_u(D_X)\\
&=0.
\end{align*}
Consequently,
$$
w_X^X(u)
=
e^{-\frac{\delta_X^X(u)^2}{2\tau^2}}
=1.
$$
It follows that every atom $\xi\in\widehat D_X$ satisfies
$$
p_X^X(\xi)
=
w_X^X(\pi_X(\xi))p_X(\xi)
=
p_X(\xi).
$$
The mass assigned to the unexplained event is
\begin{align*}
p_X^X(\partial)
&=
\sum\limits_{\xi\in\widehat D_X}
\left(1-w_X^X(\pi_X(\xi))\right)p_X(\xi)\\
&=0\\
&=p_X(\partial).
\end{align*}
Thus $p_X^X=p_X$ on the entire event space
$\widehat D_X\cup\{\partial\}$.

\medskip
\noindent\textit{Proof of (ii).}
Suppose first that $\delta_X^Y=0$ on $D_X$. Then
$$
w_X^Y(u)
=
e^{-\frac{\delta_X^Y(u)^2}{2\tau^2}}
=1
$$
for every $u\in D_X$. The same calculation as in part (i) gives
$p_X^Y(\xi)=p_X(\xi)$ for every $\xi\in\widehat D_X$ and
$p_X^Y(\partial)=p_X(\partial)=0$. Hence $p_X^Y=p_X$.

Conversely, suppose that $p_X^Y=p_X$. Every off-diagonal persistence point
has positive persistence, so $p_X(\xi)>0$ for every
$\xi\in\widehat D_X$. Therefore,
\begin{align*}
p_X^Y(\xi)=p_X(\xi)
&\quad\Longrightarrow\quad
w_X^Y(\pi_X(\xi))p_X(\xi)=p_X(\xi)\\
&\quad\Longrightarrow\quad
w_X^Y(\pi_X(\xi))=1.
\end{align*}
By the definition of $w_X^Y$,
$$
e^{-\frac{\delta_X^Y(\pi_X(\xi))^2}{2\tau^2}}=1.
$$
Since $\tau>0$ and $e^{-a}=1$ for $a\geq0$ only when $a=0$, we obtain
$\delta_X^Y(\pi_X(\xi))=0$. Every point of $D_X$ has an occurrence in
$\widehat D_X$, so $\delta_X^Y=0$ on $D_X$.

\medskip
\noindent\textit{Proof of (iii).}
For probability measures $p$ and $q$ on
$\Omega_X=\widehat D_X\cup\{\partial\}$, the total variation distance is
$$
d_{\mathrm{TV}}(p,q)
=
\sup\limits_{A\subseteq\Omega_X}|p(A)-q(A)|.
$$
For $B\subseteq\widehat D_X$, define the mass removed from $B$ by
$$
m(B)
=
\sum\limits_{\xi\in B}
\left(1-w_X^Y(\pi_X(\xi))\right)p_X(\xi).
$$
Every summand is nonnegative, and hence
$$
0
\leq
m(B)
\leq
m(\widehat D_X)
=
p_X^Y(\partial).
$$
If $\partial\notin A$, then $A\subseteq\widehat D_X$ and
\begin{align*}
p_X^Y(A)-p_X(A)
&=
\sum\limits_{\xi\in A}
\left(w_X^Y(\pi_X(\xi))-1\right)p_X(\xi)\\
&=-m(A).
\end{align*}
Therefore,
$$
|p_X^Y(A)-p_X(A)|
=m(A)
\leq
p_X^Y(\partial).
$$
If $\partial\in A$, write $A=B\cup\{\partial\}$ with
$B\subseteq\widehat D_X$. Since $p_X(\partial)=0$,
\begin{align*}
p_X^Y(A)-p_X(A)
&=
p_X^Y(\partial)+p_X^Y(B)-p_X(B)\\
&=
p_X^Y(\partial)-m(B).
\end{align*}
This value lies between $0$ and $p_X^Y(\partial)$, so the same upper bound
holds for every $A\subseteq\Omega_X$. Taking $A=\{\partial\}$ gives
\begin{align*}
|p_X^Y(\{\partial\})-p_X(\{\partial\})|
&=
|p_X^Y(\partial)-0|\\
&=
p_X^Y(\partial),
\end{align*}
and therefore the upper bound is attained. Thus
$$
d_{\mathrm{TV}}(p_X^Y,p_X)=p_X^Y(\partial).
$$

\subsection{Stability lemmas}
\label{app:alpha-wp-stability}

For the stability proofs, let $D$ be any finite persistence diagram and define
$$
\alpha_u(D)
=
\sum\limits_{v\in D}\phi_u(v).
$$
In birth--persistence coordinates, every $v\in\Delta$ has
$\ell_v=0$.

\begin{lemma}
\label{lem:alpha-wp-stability}
Let $D_X,D_Y$ be finite persistence diagrams. If
$$
L_\phi=L_g+M_gL_0,
$$
then, for every $u\in\mathcal{Z}$,
$$
|\alpha_u(D_X)-\alpha_u(D_Y)|
\leq
L_\phi W_1(D_X,D_Y).
$$
\end{lemma}

\begin{proof}
Fix $u\in\mathcal{Z}$ and a matching
$\gamma\in\Gamma(D_X,D_Y)$. Because $g(0)=0$, every diagonal point
$v\in\Delta$ satisfies
$$
\phi_u(v)
=
k(u,v)g(\ell_v)
=0.
$$
We may therefore add the matched diagonal occurrences to the sums defining
$\alpha_u(D_X)$ and $\alpha_u(D_Y)$ without changing either value. Pairs in
which both points lie on the diagonal contribute zero and may be omitted.

Consider any remaining matched pair $v$ and $v'=\gamma(v)$. Adding and
subtracting $k(u,v)g(\ell_{v'})$ gives
\begin{align*}
|\phi_u(v)-\phi_u(v')|
&=
|k(u,v)g(\ell_v)-k(u,v')g(\ell_{v'})|\\
&\leq
k(u,v)|g(\ell_v)-g(\ell_{v'})|
+g(\ell_{v'})|k(u,v)-k(u,v')|.
\end{align*}
Since $0<k\leq1$, $g\leq M_g$, $g$ is $L_g$-Lipschitz, and $k$ is
$L_0$-Lipschitz in its second argument,
\begin{align*}
|\phi_u(v)-\phi_u(v')|
&\leq
L_g|\ell_v-\ell_{v'}|
+M_gL_0\Vert v-v'\Vert\\
&\leq
\left(L_g+M_gL_0\right)\Vert v-v'\Vert\\
&=
L_\phi\Vert v-v'\Vert.
\end{align*}
The second inequality uses
$|\ell_v-\ell_{v'}|\leq\Vert v-v'\Vert$ in birth--persistence
coordinates. Summing over the matching yields
\begin{align*}
|\alpha_u(D_X)-\alpha_u(D_Y)|
&\leq
\sum\limits_{v\in D_X\cup\Delta}
|\phi_u(v)-\phi_u(\gamma(v))|\\
&\leq
L_\phi
\sum\limits_{v\in D_X\cup\Delta}
\Vert v-\gamma(v)\Vert.
\end{align*}
This bound holds for every admissible matching $\gamma$. Taking the infimum
over $\Gamma(D_X,D_Y)$ proves
$$
|\alpha_u(D_X)-\alpha_u(D_Y)|
\leq
L_\phi W_1(D_X,D_Y).
$$
\end{proof}

\medskip

\begin{lemma}
\label{lem:delta-wp-stability}
We have
$$
\Vert \delta_X^Y \Vert_\infty
\leq
L_\phi W_1(D_X,D_Y).
$$
Here the supremum is taken over $u\in D_X$.
\end{lemma}

\begin{proof}
For every $u\in D_X$, Definition~\ref{def:induced_weight} and
Lemma~\ref{lem:alpha-wp-stability} give
\begin{align*}
|\delta_X^Y(u)|
&=
|\alpha_u(D_Y)-\alpha_u(D_X)|\\
&\leq
L_\phi W_1(D_X,D_Y).
\end{align*}
The right-hand side does not depend on $u$. Taking the supremum over
$u\in D_X$ proves the claim.
\end{proof}

\medskip

\subsection{Proof of Theorem~\ref{thm:unexplained-mass-stability}}
\label{app:proof-unexplained-mass-stability}
For the first claim, Lemma~\ref{lem:delta-wp-stability} gives
$R=L_\phi W_1(D_X,D_Y)$ such that
$|\delta_X^Y(\pi_X(\xi))|\leq R$ for every $\xi\in\widehat D_X$.
By Proposition~\ref{proposition:self-recovery}(iii) and
Definition~\ref{def:induced_probability_measure},
\begin{align*}
d_{\mathrm{TV}}(p_X^Y,p_X)
&=
p_X^Y(\partial)\\
&=
\sum\limits_{\xi\in\widehat D_X}p_X(\xi)
\left(
1-e^{-\frac{\delta_X^Y(\pi_X(\xi))^2}{2\tau^2}}
\right).
\end{align*}
For each $\xi\in\widehat D_X$, the number
$$
t_\xi
=
\frac{\delta_X^Y(\pi_X(\xi))^2}{2\tau^2}
$$
is nonnegative. Applying $1-e^{-t}\leq t$ to each $t_\xi$ gives
\begin{align*}
d_{\mathrm{TV}}(p_X^Y,p_X)
&\leq
\sum\limits_{\xi\in\widehat D_X}p_X(\xi)
\frac{\delta_X^Y(\pi_X(\xi))^2}{2\tau^2}\\
&\leq
\sum\limits_{\xi\in\widehat D_X}p_X(\xi)
\frac{R^2}{2\tau^2}\\
&=
\frac{R^2}{2\tau^2}
\sum\limits_{\xi\in\widehat D_X}p_X(\xi)\\
&=
\frac{R^2}{2\tau^2}\\
&=
\frac{L_\phi^2}{2\tau^2}W_1(D_X,D_Y)^2.
\end{align*}

For the second claim, fix $D_X$ and define
$$
f(a)=e^{-\frac{a^2}{2\tau^2}}.
$$
Differentiating gives
$$
f'(a)
=
-\frac{a}{\tau^2}e^{-\frac{a^2}{2\tau^2}}.
$$
If $x=\frac{|a|}{\tau}$, then
$$
|f'(a)|
=
\frac{1}{\tau}xe^{-\frac{x^2}{2}}.
$$
The derivative of $xe^{-\frac{x^2}{2}}$ on $[0,\infty)$ is
$$
(1-x^2)e^{-\frac{x^2}{2}},
$$
so the maximum occurs at $x=1$ and equals $e^{-\frac{1}{2}}$. Therefore,
$$
\sup\limits_{a\in\mathbb{R}}|f'(a)|
=
\frac{1}{\tau\sqrt e}.
$$
Moreover, Lemma~\ref{lem:alpha-wp-stability} gives, for every
$u\in\mathcal{Z}$,
\begin{align*}
|\delta_X^Y(u)-\delta_X^{Y'}(u)|
&=
\left|
\alpha_u(D_Y)-\alpha_u(D_X)
-\alpha_u(D_{Y'})+\alpha_u(D_X)
\right|\\
&=
|\alpha_u(D_Y)-\alpha_u(D_{Y'})|\\
&\leq
L_\phi W_1(D_Y,D_{Y'}).
\end{align*}
The mean value theorem applied to $f$ gives
\begin{align*}
|w_X^Y(u)-w_X^{Y'}(u)|
&=
|f(\delta_X^Y(u))-f(\delta_X^{Y'}(u))|\\
&\leq
\frac{1}{\tau\sqrt e}
|\delta_X^Y(u)-\delta_X^{Y'}(u)|\\
&\leq
\frac{L_\phi}{\tau\sqrt e}W_1(D_Y,D_{Y'}).
\end{align*}
Therefore, for every $\xi\in\widehat D_X$,
\begin{align*}
|p_X^Y(\xi)-p_X^{Y'}(\xi)|
&=
p_X(\xi)
|w_X^Y(\pi_X(\xi))-w_X^{Y'}(\pi_X(\xi))|\\
&\leq
p_X(\xi)\frac{L_\phi}{\tau\sqrt e}W_1(D_Y,D_{Y'}).
\end{align*}
Let
$$
S
=
\sum\limits_{\xi\in\widehat D_X}
|p_X^Y(\xi)-p_X^{Y'}(\xi)|.
$$
Summing the preceding pointwise bound and using
$\sum\limits_{\xi\in\widehat D_X}p_X(\xi)=1$ gives
$$
S
\leq
\frac{L_\phi}{\tau\sqrt e}W_1(D_Y,D_{Y'}).
$$
The difference between the two boundary masses satisfies
\begin{align*}
|p_X^Y(\partial)-p_X^{Y'}(\partial)|
&=
\left|
\sum\limits_{\xi\in\widehat D_X}p_X(\xi)
\left(
w_X^{Y'}(\pi_X(\xi))-w_X^Y(\pi_X(\xi))
\right)
\right|\\
&\leq
\sum\limits_{\xi\in\widehat D_X}p_X(\xi)
|w_X^{Y'}(\pi_X(\xi))-w_X^Y(\pi_X(\xi))|\\
&=
S.
\end{align*}
For probability measures on a finite event space, total variation is one half
of their $\ell^1$ difference. Since both measures are defined on
$\Omega_X$, we obtain
\begin{align*}
d_{\mathrm{TV}}(p_X^Y,p_X^{Y'})
&=
\frac{1}{2}
\left(
S+|p_X^Y(\partial)-p_X^{Y'}(\partial)|
\right)\\
&\leq
S\\
&\leq
\frac{L_\phi}{\tau\sqrt e}W_1(D_Y,D_{Y'}).
\end{align*}
Thus both claims hold with, for example,
$$
C_X
=
\max\left\{
\frac{L_\phi^2}{2\tau^2},
\frac{L_\phi}{\tau\sqrt e}
\right\}.
$$
\subsection{Proof of Proposition~\ref{proposition:cross_entropy_kl_decomposition}}
\label{app:proof-cross-entropy-kl}
For every $\xi\in\widehat D_X$,
Definition~\ref{def:induced_probability_measure} and the positivity of
$p_X(\xi)$ give
\begin{align*}
\frac{p_X(\xi)}{p_X^Y(\xi)}
&=
\frac{p_X(\xi)}{w_X^Y(\pi_X(\xi))p_X(\xi)}\\
&=
\frac{1}{w_X^Y(\pi_X(\xi))}\\
&=
e^{\frac{\delta_X^Y(\pi_X(\xi))^2}{2\tau^2}}.
\end{align*}
Taking the logarithm gives
$$
\log\frac{p_X(\xi)}{p_X^Y(\xi)}
=
\frac{\delta_X^Y(\pi_X(\xi))^2}{2\tau^2}.
$$

The KL divergence on $\Omega_X$ is
$$
\operatorname{KL}(p_X\Vert p_X^Y)
=
\sum\limits_{\omega\in\Omega_X}
p_X(\omega)
\log\frac{p_X(\omega)}{p_X^Y(\omega)}.
$$
The boundary atom contributes zero because $p_X(\partial)=0$, using the
standard convention that $0\log(0/q)=0$. Thus
\begin{align*}
\operatorname{KL}(p_X\Vert p_X^Y)
&=
\sum\limits_{\xi\in\widehat D_X}
p_X(\xi)
\log\frac{p_X(\xi)}{p_X^Y(\xi)}\\
&=
\sum\limits_{\xi\in\widehat D_X}
p_X(\xi)
\frac{\delta_X^Y(\pi_X(\xi))^2}{2\tau^2}\\
&=
\frac{1}{2\tau^2}
\mathbb{E}_{\xi\sim p_X}
\left[\delta_X^Y(\pi_X(\xi))^2\right].
\end{align*}
On the other hand, expanding the logarithm of the ratio gives
\begin{align*}
\operatorname{KL}(p_X\Vert p_X^Y)
&=
\sum\limits_{\xi\in\widehat D_X}p_X(\xi)\log p_X(\xi)
-
\sum\limits_{\xi\in\widehat D_X}p_X(\xi)\log p_X^Y(\xi)\\
&=
-H(p_X)+H(p_X,p_X^Y)\\
&=
H(p_X,p_X^Y)-H(p_X).
\end{align*}
Combining the two calculations proves every equality in the proposition. The
final quantity is nonnegative because it is a positive constant times an
expectation of a square.
Suppressing the occurrence labels recovers the expectation over $u\sim p_X$
in Proposition~\ref{proposition:cross_entropy_kl_decomposition}.
\subsection{Proof of Theorem~\ref{thm:persistent-cross-entropy-wasserstein-bound}}
\label{app:proof-pce-wasserstein-bound}
Proposition~\ref{proposition:cross_entropy_kl_decomposition} gives
\begin{align*}
\Delta H_X^Y
&=
\frac{1}{2\tau^2}
\mathbb{E}_{\xi\sim p_X}
\left[\delta_X^Y(\pi_X(\xi))^2\right]\\
&=
\frac{1}{2\tau^2}
\sum\limits_{\xi\in\widehat D_X}
p_X(\xi)\delta_X^Y(\pi_X(\xi))^2.
\end{align*}
Every summand is nonnegative, so $\Delta H_X^Y\geq0$. Moreover,
$$
\delta_X^Y(\pi_X(\xi))^2
\leq
\Vert\delta_X^Y\Vert_\infty^2
$$
for every $\xi\in\widehat D_X$. Therefore,
\begin{align*}
\Delta H_X^Y
&\leq
\frac{1}{2\tau^2}
\sum\limits_{\xi\in\widehat D_X}
p_X(\xi)\Vert\delta_X^Y\Vert_\infty^2\\
&=
\frac{1}{2\tau^2}
\Vert\delta_X^Y\Vert_\infty^2
\sum\limits_{\xi\in\widehat D_X}p_X(\xi)\\
&=
\frac{1}{2\tau^2}\Vert\delta_X^Y\Vert_\infty^2.
\end{align*}
Lemma~\ref{lem:delta-wp-stability} now gives
\begin{align*}
\Delta H_X^Y
&\leq
\frac{1}{2\tau^2}
\left(L_\phi W_1(D_X,D_Y)\right)^2\\
&=
\frac{L_\phi^2}{2\tau^2}W_1(D_X,D_Y)^2.
\end{align*}
Thus the theorem holds with
$$
C=\frac{L_\phi^2}{2\tau^2}.
$$
The constants $L_\phi$ and $\tau$ are fixed independently of $D_X$ and
$D_Y$, so the same $C$ applies to every pair of diagrams.
\subsection{Proof of Theorem~\ref{thm:pce-joint-local-stability}}
\label{app:proof-pce-joint-local-stability}
For a finite persistence diagram $D$, define
$$
T(D)=\sum\limits_{z\in D}\ell_z,
\qquad
G(D)=\sum\limits_{z\in D}g(\ell_z).
$$
Set
$$
T_*=\min\{T(D_X),T(D_{X'})\},
$$
$$
B_*=\max\{
\Vert \delta_X^Y \Vert_\infty,
\Vert \delta_{X'}^{Y'} \Vert_\infty
\},
$$
and
$$
G_*=\max\{
G(D_X)+G(D_Y),
G(D_{X'})+G(D_{Y'})
\}.
$$
These quantities are finite, and $T_*>0$ because the diagrams are nonempty
and contain only off-diagonal points.

We first control the change in $\delta$ when its evaluation point also moves.
Definition~\ref{def:similarity_function} gives
$$
k(u,v)
=
h(\rho(u,v))
=
h(\rho(v,u))
=
k(v,u),
$$
so $k$ is symmetric. Its Lipschitz condition in the second argument therefore
also gives
$$
|k(u,v)-k(u',v)|
=
|k(v,u)-k(v,u')|
\leq
L_0\Vert u-u'\Vert.
$$
For any point $z\in D$, we consequently have
\begin{align*}
|\phi_u(z)-\phi_{u'}(z)|
&=
g(\ell_z)|k(u,z)-k(u',z)|\\
&\leq
L_0g(\ell_z)\Vert u-u'\Vert.
\end{align*}
Summing this inequality over $z\in D$ gives
\begin{align*}
|\alpha_u(D)-\alpha_{u'}(D)|
&=
\left|
\sum\limits_{z\in D}
\left(\phi_u(z)-\phi_{u'}(z)\right)
\right|\\
&\leq
\sum\limits_{z\in D}|\phi_u(z)-\phi_{u'}(z)|\\
&\leq
L_0G(D)\Vert u-u'\Vert.
\end{align*}

We use this pointwise bound together with
Lemma~\ref{lem:alpha-wp-stability}. For the $Y$ diagrams,
\begin{align*}
|\alpha_u(D_Y)-\alpha_{u'}(D_{Y'})|
&\leq
|\alpha_u(D_Y)-\alpha_u(D_{Y'})|
+|\alpha_u(D_{Y'})-\alpha_{u'}(D_{Y'})|\\
&\leq
L_\phi W_1(D_Y,D_{Y'})
+L_0G(D_{Y'})\Vert u-u'\Vert.
\end{align*}
The same argument gives
$$
|\alpha_u(D_X)-\alpha_{u'}(D_{X'})|
\leq
L_\phi W_1(D_X,D_{X'})
+L_0G(D_{X'})\Vert u-u'\Vert.
$$
Since $\delta_X^Y(u)=\alpha_u(D_Y)-\alpha_u(D_X)$, the triangle inequality
now yields
\begin{align*}
|\delta_X^Y(u)-\delta_{X'}^{Y'}(u')|
&\leq
L_\phi
\left(
W_1(D_X,D_{X'})+W_1(D_Y,D_{Y'})
\right)\\
&\quad+
L_0
\left(G(D_{X'})+G(D_{Y'})\right)
\Vert u-u'\Vert\\
&\leq
L_\phi
\left(
W_1(D_X,D_{X'})+W_1(D_Y,D_{Y'})
\right)
+L_0G_*\Vert u-u'\Vert.
\end{align*}

Take optimal matchings between $D_X$ and $D_{X'}$, and between $D_Y$ and
$D_{Y'}$, allowing points to be matched to the diagonal. Write the coordinate
pairs in the first matching as $(u_i,u_i')$, and assign persistence zero to a
diagonal point. Let
$$
R_X=W_1(D_X,D_{X'}),
\qquad
R_Y=W_1(D_Y,D_{Y'}),
$$
and set
$$
a_i=\frac{\ell_{u_i}}{T(D_X)},
\qquad
a_i'=\frac{\ell_{u_i'}}{T(D_{X'})}.
$$
Because persistence is the distance to the diagonal in the chosen
birth--persistence coordinates, the second-coordinate difference satisfies
$$
|\ell_{u_i}-\ell_{u_i'}|
\leq
\Vert u_i-u_i'\Vert
$$
for every matched pair, including a pair with one point on the diagonal.
Summing over the active matching pairs gives
$$
\sum\limits_i|\ell_{u_i}-\ell_{u_i'}|
\leq
\sum\limits_i\Vert u_i-u_i'\Vert
=
R_X.
$$
The total persistences satisfy
\begin{align*}
|T(D_X)-T(D_{X'})|
&=
\left|
\sum\limits_i\ell_{u_i}-\sum\limits_i\ell_{u_i'}
\right|\\
&\leq
\sum\limits_i|\ell_{u_i}-\ell_{u_i'}|\\
&\leq
R_X.
\end{align*}

We next bound the change in the normalized persistence weights. For each
pair,
\begin{align*}
|a_i-a_i'|
&=
\left|
\frac{\ell_{u_i}}{T(D_X)}
-
\frac{\ell_{u_i'}}{T(D_{X'})}
\right|\\
&\leq
\frac{|\ell_{u_i}-\ell_{u_i'}|}{T(D_X)}
+
\ell_{u_i'}
\left|
\frac{1}{T(D_X)}-\frac{1}{T(D_{X'})}
\right|.
\end{align*}
Summing over $i$ and using
$\sum\limits_i\ell_{u_i'}=T(D_{X'})$ gives
\begin{align*}
\sum\limits_i|a_i-a_i'|
&\leq
\frac{1}{T(D_X)}
\sum\limits_i|\ell_{u_i}-\ell_{u_i'}|
+
T(D_{X'})
\frac{|T(D_X)-T(D_{X'})|}{T(D_X)T(D_{X'})}\\
&=
\frac{1}{T(D_X)}
\sum\limits_i|\ell_{u_i}-\ell_{u_i'}|
+
\frac{|T(D_X)-T(D_{X'})|}{T(D_X)}\\
&\leq
\frac{2R_X}{T(D_X)}\\
&\leq
\frac{2R_X}{T_*}.
\end{align*}
In particular,
$$
\sum\limits_i a_i
=
\sum\limits_i a_i'
=1.
$$

If $u_i$ is off the diagonal, set
$$
r_i=\delta_X^Y(u_i)^2;
$$
if $u_i$ is on the diagonal, set $r_i=0$. Define $r_i'$ in the same way for
$u_i'$. These zero values do not change either expectation because the
corresponding normalized persistence weight is also zero. We then have
$$
0\leq r_i\leq B_*^2,
\qquad
0\leq r_i'\leq B_*^2.
$$

Consider an index $i$ for which both points are off the diagonal. Factoring
the difference of squares gives
\begin{align*}
|r_i-r_i'|
&=
\left|
\delta_X^Y(u_i)^2-\delta_{X'}^{Y'}(u_i')^2
\right|\\
&=
\left|
\delta_X^Y(u_i)-\delta_{X'}^{Y'}(u_i')
\right|
\left|
\delta_X^Y(u_i)+\delta_{X'}^{Y'}(u_i')
\right|\\
&\leq
2B_*
\left|
\delta_X^Y(u_i)-\delta_{X'}^{Y'}(u_i')
\right|\\
&\leq
2B_*L_\phi(R_X+R_Y)
+2B_*L_0G_*\Vert u_i-u_i'\Vert.
\end{align*}

To compare the weighted terms, let $c_i=\min\{a_i,a_i'\}$. Then
\begin{align*}
a_ir_i-a_i'r_i'
&=
(a_i-c_i)r_i
-(a_i'-c_i)r_i'
+c_i(r_i-r_i').
\end{align*}
At most one of $a_i-c_i$ and $a_i'-c_i$ is nonzero. Consequently,
\begin{align*}
|a_ir_i-a_i'r_i'|
&\leq
B_*^2
\left(
|a_i-c_i|+|a_i'-c_i|
\right)
+c_i|r_i-r_i'|\\
&=
B_*^2|a_i-a_i'|+c_i|r_i-r_i'|.
\end{align*}
If one point is on the diagonal, then $c_i=0$, so the same inequality holds
without requiring a comparison between two $r$ values.

Summing over all active matching pairs gives
\begin{align*}
\left|
\sum\limits_i a_ir_i-
\sum\limits_i a_i'r_i'
\right|
&\leq
\sum\limits_i|a_ir_i-a_i'r_i'|\\
&\leq
B_*^2\sum\limits_i|a_i-a_i'|
+
\sum\limits_i c_i|r_i-r_i'|.
\end{align*}
Since $0\leq c_i\leq a_i$ and $\sum\limits_i a_i=1$,
$$
\sum\limits_i c_i\leq1.
$$
Also, $c_i\leq1$ for every $i$, so
$$
\sum\limits_i c_i\Vert u_i-u_i'\Vert
\leq
\sum\limits_i\Vert u_i-u_i'\Vert
=
R_X.
$$
Using these two bounds, the estimate for $|r_i-r_i'|$, and the normalization
bound above, we obtain the following. The $|r_i-r_i'|$ estimate is needed
only when both points are off the diagonal; every other term has $c_i=0$.
\begin{align*}
\left|
\sum\limits_i a_ir_i-
\sum\limits_i a_i'r_i'
\right|
&\leq
B_*^2\frac{2R_X}{T_*}
+2B_*L_\phi(R_X+R_Y)
+2B_*L_0G_*R_X\\
&=
\left(
\frac{2B_*^2}{T_*}
+2B_*(L_\phi+L_0G_*)
\right)R_X
+2B_*L_\phi R_Y.
\end{align*}
The matching includes each off-diagonal occurrence exactly once. Therefore,
\begin{align*}
\sum\limits_i a_ir_i
&=
\mathbb{E}_{u\sim p_X}\left[\delta_X^Y(u)^2\right],\\
\sum\limits_i a_i'r_i'
&=
\mathbb{E}_{u'\sim p_{X'}}\left[\delta_{X'}^{Y'}(u')^2\right].
\end{align*}
Proposition~\ref{proposition:cross_entropy_kl_decomposition} consequently
gives
\begin{align*}
|\Delta H_X^Y-\Delta H_{X'}^{Y'}|
&=
\frac{1}{2\tau^2}
\left|
\sum\limits_i a_ir_i-
\sum\limits_i a_i'r_i'
\right|\\
&\leq
C_XR_X+C_YR_Y,
\end{align*}
where
$$
C_X
=
\frac{1}{\tau^2}
\left(
\frac{B_*^2}{T_*}+B_*(L_\phi+L_0G_*)
\right),
\qquad
C_Y=\frac{B_*L_\phi}{\tau^2}.
$$
To make the constant uniform on a neighborhood of $(D_X,D_Y)$, first note
that $g(0)=0$ allows us to extend $g(\ell)$ by zero to diagonal points. For
any optimal matching $\gamma\in\Gamma(D,D')$, the Lipschitz condition on
$g$ gives
\begin{align*}
|G(D)-G(D')|
&=
\left|
\sum\limits_{z\in D\cup\Delta}
\left(g(\ell_z)-g(\ell_{\gamma(z)})\right)
\right|\\
&\leq
\sum\limits_{z\in D\cup\Delta}
|g(\ell_z)-g(\ell_{\gamma(z)})|\\
&\leq
L_g
\sum\limits_{z\in D\cup\Delta}
|\ell_z-\ell_{\gamma(z)}|\\
&\leq
L_g
\sum\limits_{z\in D\cup\Delta}
\Vert z-\gamma(z)\Vert\\
&=
L_gW_1(D,D').
\end{align*}
Choose
$$
r_{X,Y}=\frac{T(D_X)}{2}
$$
and suppose that $R_X+R_Y\leq r_{X,Y}$. Then
$$
R_X
\leq
R_X+R_Y
\leq
\frac{T(D_X)}{2}.
$$
Since $|T(D_X)-T(D_{X'})|\leq R_X$, it follows that
\begin{align*}
T(D_{X'})
&\geq
T(D_X)-R_X\\
&\geq
\frac{T(D_X)}{2}.
\end{align*}
The same lower bound is immediate for $T(D_X)$ itself. Hence
$$
T_*
=
\min\{T(D_X),T(D_{X'})\}
\geq
\frac{T(D_X)}{2}.
$$

We next bound $G_*$. The Lipschitz estimate for $G$ gives
\begin{align*}
G(D_{X'})+G(D_{Y'})
&\leq
G(D_X)+L_gR_X+G(D_Y)+L_gR_Y\\
&=
G(D_X)+G(D_Y)+L_g(R_X+R_Y)\\
&\leq
G(D_X)+G(D_Y)+L_gr_{X,Y}.
\end{align*}
Define
$$
\overline G_{X,Y}
=
G(D_X)+G(D_Y)+L_gr_{X,Y}.
$$
Since $\overline G_{X,Y}$ also bounds the unprimed sum
$G(D_X)+G(D_Y)$, we have
$$
G_*
\leq
\overline G_{X,Y}.
$$

Finally, $0<k(u,z)\leq1$ implies
\begin{align*}
0
\leq
\alpha_u(D)
&=
\sum\limits_{z\in D}k(u,z)g(\ell_z)\\
&\leq
\sum\limits_{z\in D}g(\ell_z)\\
&=
G(D).
\end{align*}
Therefore,
\begin{align*}
|\delta_X^Y(u)|
&=
|\alpha_u(D_Y)-\alpha_u(D_X)|\\
&\leq
G(D_X)+G(D_Y),
\end{align*}
and the same argument applies to $\delta_{X'}^{Y'}$. Consequently,
$$
B_*
\leq
G_*
\leq
\overline G_{X,Y}.
$$
Substituting the bounds for $T_*$, $B_*$, and $G_*$ into $C_X$ gives
\begin{align*}
C_X
&\leq
\frac{1}{\tau^2}
\left(
\frac{\overline G_{X,Y}^2}{T(D_X)/2}
+\overline G_{X,Y}
\left(L_\phi+L_0\overline G_{X,Y}\right)
\right)\\
&=
\frac{1}{\tau^2}
\left(
\frac{2\overline G_{X,Y}^2}{T(D_X)}
+\overline G_{X,Y}
\left(L_\phi+L_0\overline G_{X,Y}\right)
\right).
\end{align*}
Similarly,
$$
C_Y
\leq
\frac{\overline G_{X,Y}L_\phi}{\tau^2}.
$$
Both quantities are therefore bounded by the single finite constant
$$
C_{X,Y}
=
\frac{1}{\tau^2}
\left(
\frac{2\overline G_{X,Y}^2}{T(D_X)}
+\overline G_{X,Y}
\left(L_\phi+L_0\overline G_{X,Y}\right)
\right).
$$
The constant $C_{X,Y}$ depends only on the fixed center diagrams
$D_X,D_Y$ and on the fixed functions and scales. It does not depend on
$D_{X'}$ or $D_{Y'}$ as long as
$R_X+R_Y\leq r_{X,Y}$. Hence
\begin{align*}
|\Delta H_X^Y-\Delta H_{X'}^{Y'}|
&\leq
C_XR_X+C_YR_Y\\
&\leq
C_{X,Y}(R_X+R_Y)\\
&=
C_{X,Y}
\left(
W_1(D_X,D_{X'})+W_1(D_Y,D_{Y'})
\right),
\end{align*}
which proves the joint local stability claim.
If $D_{X'}=D_X$, then $R_X=0$, and the same estimate reduces to
$$
|\Delta H_X^Y-\Delta H_X^{Y'}|
\leq
C_{X,Y}W_1(D_Y,D_{Y'}),
$$
which proves the final claim.
\section{Additional Experimental Details}
\label{app:additional-experiments}

For every $u\in D_X$, the numerical implementation evaluates
$$
\alpha_u(D_Y)
=
\sum\limits_{v\in D_Y}k(u,v)g(\ell_v),
\qquad
\delta_X^Y(u)
=
\alpha_u(D_Y)-\alpha_u(D_X),
$$
then sets $p_X^Y(u)=e^{-\frac{\delta_X^Y(u)^2}{2\tau^2}}p_X(u)$ and assigns the
remaining mass to $\partial$. No further normalization is required. Every
experiment below uses
$$
g(t)=\frac{t}{1+t},
$$
which satisfies the regularity assumptions in
Appendix~\ref{app:g-regularity}. The reference probability itself remains the
standard persistence probability
$p_X(u)=\ell_u/\sum\limits_{v\in D_X}\ell_v$.

\subsection{Two-loop configuration}

The point clouds in Figure~\ref{fig:two-loops-equal-pe} are generated with
random seed $7$. Their $H_1$ diagrams are computed from Vietoris--Rips
filtrations, and features with persistence below $10^{-3}$ are discarded.
The Gaussian similarity is evaluated in birth--persistence coordinates with
$\sigma_b=\sigma_\ell=0.50$. A single response scale is used for all four
comparisons:
$$
\tau=0.298499.
$$
It is the median of the positive values
$|\delta_X^{Y_i}(u)|$, pooled over every $u\in D_X$ and
$i\in\{1,2,3,4\}$. It is calibrated once and is not adjusted separately for
individual pairs.

Table~\ref{tab:two-loop-details} lists the sizes of the point clouds and
diagrams together with the persistent entropy values before rounding.

\begin{table}[h]
\centering
\caption{Two-loop data and persistence-diagram summary.}
\label{tab:two-loop-details}
\begin{tabular}{lrrr}
\toprule
Data & Point-cloud size & Number of $H_1$ points & Persistent entropy \\
\midrule
$X$   & 239 & 12 & 1.500261 \\
$Y_1$ & 758 & 64 & 1.500305 \\
$Y_2$ & 695 & 70 & 1.500447 \\
$Y_3$ & 605 & 59 & 1.500307 \\
$Y_4$ &  84 & 13 & 1.500474 \\
\bottomrule
\end{tabular}
\end{table}

The augmented probabilities satisfy the numerical identities expected from
the construction: the mass on $D_X$ plus the mass at $\partial$ equals
one, total variation equals the boundary mass, and entropy excess equals
$\frac{1}{2\tau^2}\mathbb{E}_{u\sim p_X}[\delta_X^Y(u)^2]$, up to floating-point
precision.

\subsection{Spring--mass configuration}
\label{app:spring-mass-configuration}

We use double-precision arithmetic and integrate the spring--mass equations by
fixed-step fourth-order Runge--Kutta. The time step is $0.01$, the initial
state is $(x,\dot{x},y,\dot{y})=(1,0,0,1)$, and the solver performs
$10{,}000$ updates. We retain every tenth post-update state, giving
$1{,}000$ observations at an effective sampling interval of $0.1$. We
discard no transient and add no observation noise. The coupling parameters
range over the $9\times9$ grid
$\alpha,\beta\in\{0,0.1,\ldots,0.8\}$. Here $\alpha$ controls
$B\to A$ and $\beta$ controls $A\to B$, giving one independent case,
eight cases in each one-way direction, and $64$ bidirectional cases.

Write the sampled observations as $x_t$ and $y_t$, and let $q$ denote
the delay lag, to distinguish it from the PCE response scale $\tau$. We use
the raw observations without centering or rescaling. The individual
reconstructions are
$$
X_A
=
\left\{
(x_t,x_{t+q},\ldots,x_{t+(E-1)q})
\right\}_t,
\qquad
X_B
=
\left\{
(y_t,y_{t+q},\ldots,y_{t+(E-1)q})
\right\}_t.
$$
For even $E$, the multivariate reconstruction of
\citet{cao1998dynamics} combines half of the coordinates from each series:
$$
X_{AB}
=
\left\{
\bigl(
x_t,\ldots,x_{t+(E/2-1)q},
y_t,\ldots,y_{t+(E/2-1)q}
\bigr)
\right\}_t.
$$
We set $E=4$ and $q=1$, where one lag is one sampled time step, or $0.1$
in physical time. Thus the individual points have four delay coordinates and
the joint points are $(x_t,x_{t+1},y_t,y_{t+1})$. The individual
reconstructions contain $997$ valid points and the joint reconstruction
initially contains $999$. We retain the first $997$ common time indices in
all three clouds; no additional point subsampling is performed.

We construct alpha complexes in GUDHI 3.11.0~\citep{maria2014gudhi} and
compute their finite $H_1$ persistence intervals. GUDHI returns squared
alpha-filtration radii, so we take their square roots before forming the
diagrams. We then convert
each birth--death point to $z=(b,\ell)$, where $\ell=d-b$, and divide both
coordinates in each diagram by that diagram's maximum persistence. Thus
$\max_{z\in D}\ell_z=1$ separately for $D_A,D_B$, and $D_{AB}$. No
minimum-persistence threshold is applied. These choices follow the procedure
reported by \citet{DBLP:journals/jsiaml/BandoKY22}. Because their original
implementation is not publicly available and some low-level numerical
choices are not specified, we follow the published setup but do not claim an
exact numerical reproduction of their Table~1.

PCE is computed directly from $D_A$ and $D_B$; it does not use
$D_{AB}$. We use $g(t)=\frac{t}{1+t}$ and the anisotropic Gaussian
similarity
$$
\kappa(u,v)
=
e^{-\frac{1}{2}
\left[
\left(\frac{b_u-b_v}{\sigma_b}\right)^2
+
\left(\frac{\ell_u-\ell_v}{\sigma_\ell}\right)^2
\right]},
$$
with
$$
\sigma_b=0.40,
\qquad
\sigma_\ell=0.15,
\qquad
\tau=2.001779.
$$
We first compute all local differences, pool
$\lvert\delta_A^B\rvert$ and $\lvert\delta_B^A\rvert$ over the
$81$ coupling settings, and set $\tau$ to their global median. The scale
is therefore selected once for the complete experiment and is never
recalibrated for an individual diagram pair. The entropy excess and
unexplained mass are then computed from the same differences using the
augmented probability in Definition~\ref{def:induced_probability_measure}.

Every symmetric baseline compares $D_{AB}$ separately with $D_A$ and
$D_B$. We compute the exact bottleneck distance and the $2$-Wasserstein
distance in GUDHI, using the $\ell_\infty$ ground metric and allowing
matching to the diagonal. Betti curves are sampled at $256$ equally spaced
points over $[0,1.8]$. Persistence landscapes use their first five layers
on the same grid. Persistence silhouettes use uniform weighting and are
sampled at $100$ points over $[0,1.8]$, following the default GUDHI
settings for weighting and resolution. For these three functional
representations, we report the Euclidean difference between sampled vectors
divided by the square root of the vector length.

Persistence images are formed in birth--persistence coordinates on a
$32\times32$ grid over $[0,1.4]^2$. Each point contributes an isotropic
Gaussian with width $0.08$, weighted by its persistence; image differences
use the same length-normalized Euclidean distance. For the persistence
scale-space kernel, we use
$$
k_\sigma(D_1,D_2)
=
\frac{1}{8\pi\sigma}
\sum\limits_{u\in D_1}
\sum\limits_{v\in D_2}
\left(
e^{-\frac{\Vert u-v \Vert_2^2}{8\sigma}}
-
e^{-\frac{\Vert u-\bar v \Vert_2^2}{8\sigma}}
\right),
$$
where $u,v,\bar v$ are in birth--death coordinates and $\bar v$ is the
reflection of $v$ across the diagonal. The corresponding distance is
$\bigl(k_\sigma(D_1,D_1)+k_\sigma(D_2,D_2)
-2k_\sigma(D_1,D_2)\bigr)^{1/2}$. We pool all finite diagram points from
the $81$ coupling settings and set
$\sigma=\frac{1}{8}\operatorname{median}_{u\ne v}\Vert u-v \Vert_2^2
=0.016325$, using at most $250{,}000$ pairs with seed $7$.

In Figure~\ref{fig:spring-mass}, the horizontal coordinate is the comparison
with $D_A$ and the vertical coordinate is the comparison with $D_B$ for
every symmetric baseline. The PCE panels instead plot the two directed
quantities computed from $D_A$ and $D_B$, as stated in the figure caption.
All parameters above are fixed over the complete $9\times9$ grid.

\subsection{Knowledge-distillation configuration}
\label{app:knowledge-distillation-configuration}

The knowledge-distillation experiment uses CIFAR-100 with $50{,}000$
training images and $10{,}000$ test images. We use the pretrained ResNet56
teacher and the ResNet20 student architecture from the TopKD implementation
\citep{kim2024topkd}. All three reported runs start from the same student
weights and use batch size $64$, seed $7$, and $240$ epochs of SGD with
initial learning rate $0.05$, momentum $0.9$, and weight decay
$5\times10^{-4}$. The learning rate is multiplied by $0.1$ after epochs
$150$, $180$, and $210$. Standard random cropping with four-pixel padding
and random horizontal flipping are used during training. The KD temperature is
$4$, and its coefficient is $2$. TopKD uses topology coefficient $5$, so
its complete objective is
$\mathcal L_{\mathrm{CE}}+2\mathcal L_{\mathrm{KD}}
+5\mathcal L_{\mathrm{Top}}$.

For a mini-batch of $64$ images, let $F_T$ and $F_S$ denote the final
teacher and student feature clouds. Each feature vector is $\ell_2$-normalized.
We compute exact finite $H_0$ Vietoris--Rips persistence by a minimum
spanning tree. Its $63$ edge lengths are the finite persistence lifetimes;
all births are zero. The teacher features and $D_T$ are detached, while the
student diagram remains differentiable. The teacher persistence probability
is
$$
p_T(t_i)
=
\frac{t_i}{\sum\limits_j t_j}.
$$
For lifetimes $t$ and $s$, we use
$$
\kappa(t,s)
=
e^{-\frac{(t-s)^2}{2\eta^2}},
\qquad
g(t)=\frac{t}{1+t}.
$$
For each of the first eight training mini-batches, we compute the median
nonzero pairwise gap between teacher lifetimes. Their median defines
$\eta$. After fixing $\eta$, $\tau$ is the median positive value of
$\lvert\delta_T^S\rvert$ pooled over the same diagram pairs. Calibration is
performed once before training and gives
$$
\eta=0.074140072,
\qquad
\tau=16.062859.
$$

For EM-PCE, define the total explained mass on the teacher diagram by
$$
r_T^S
=
\sum\limits_{u\in D_T}p_T(u)w_T^S(u)
=
1-p_T^S(\partial).
$$
EM-PCE removes the unexplained event and normalizes the explained masses on
$D_T$:
$$
p_T^{S,\mathrm{EM}}(u)
=
\frac{p_T(u)w_T^S(u)}{r_T^S},
\qquad
u\in D_T.
$$
The corresponding entropy excess is
\begin{align*}
\Delta H_T^{S,\mathrm{EM}}
&=
H(p_T,p_T^{S,\mathrm{EM}})-H(p_T)\\
&=
\frac{1}{2\tau^2}
\mathbb{E}_{u\sim p_T}
\left[\delta_T^S(u)^2\right]
+\log r_T^S.
\end{align*}
This normalization compares the relative allocation of probability among the
explained teacher atoms. It is used only as an experimental comparison and is
not the augmented probability in Definition~\ref{def:induced_probability_measure}.
The EM-PCE run uses the fixed coefficient $5$ in its training objective.

At every training mini-batch, we compute the gradients of
$5\mathcal L_{\mathrm{Top}}$ and $\Delta H_T^S$ with respect to $F_S$.
The batch-wise coefficient is
$$
c_{\mathrm{batch}}
=
\frac{
\Vert\nabla_{F_S}\left(5\mathcal L_{\mathrm{Top}}\right)\Vert_2
}{
\Vert\nabla_{F_S}\Delta H_T^S\Vert_2
}.
$$
In the implementation, both norms are bounded below by $10^{-12}$ for
numerical stability. We detach $c_{\mathrm{batch}}$ before differentiating the
training loss.
Consequently,
$$
\Vert
\nabla_{F_S}\left(c_{\mathrm{batch}}\Delta H_T^S\right)
\Vert_2
=
\Vert
\nabla_{F_S}\left(5\mathcal L_{\mathrm{Top}}\right)
\Vert_2
$$
up to numerical precision, without changing the PCE gradient direction. The
TopKD topology term is removed from the PCE training objective, but its frozen
RipsNet remains in the computation only to provide this reference norm.

Before official test evaluation, the PCE configuration was selected on a
deterministic $45{,}000/5{,}000$ split of the CIFAR-100 training data; the
official test set was not loaded in this selection run. On this split, TopKD
and PCE reached best validation accuracies of $69.62\%$ and $70.64\%$,
respectively, with last-$10$-epoch means of $69.318\%$ and $70.336\%$.
The configuration was then fixed for the full $50{,}000$-image training run.
The reported TopKD, PCE, and EM-PCE runs share the student initialization,
data order, augmentation, optimizer, learning-rate schedule, CE loss, and KD
loss.

\end{document}